\documentclass[sigconf]{acmart}
\usepackage{booktabs} 
\usepackage{dcolumn}
\usepackage{graphicx}
\usepackage{amssymb,amsmath}
\usepackage{multirow}
\usepackage{url}
\usepackage[algo2e,linesnumbered,ruled]{algorithm2e}
\usepackage{tikz, pgfplots}
\newtheorem{definition}{Definition}
\newtheorem{criterion}{Criterion}

\usetikzlibrary{positioning,shapes,arrows}
\newcolumntype{M}[1]{D{.}{.}{1.#1}}
\pgfplotsset{width=0.2 \textwidth,compat=1.8}
\pgfmathdeclarefunction{gauss}{2}{%
	\pgfmathparse{1/(#2*sqrt(2*pi))*exp(-((x-#1)^2)/(2*#2^2))}%
}

\def\eg{e.g.,~}                

\newlength\paramargin
\newlength\figmargin
\newlength\secmargin

\newcolumntype{?}{!{\vrule width 1.25pt}}

\usepackage{color}
\long\def\ignorethis#1{}

\DeclareMathOperator\erf{erf}

\renewcommand\footnotetextcopyrightpermission[1]{}

\setcopyright{rightsretained}

\acmDOI{10.475/123_4}

\acmISBN{123-4567-24-567/08/06}

\acmConference[KDD 2018]{ACM SIGKDD Conference on Knowledge Discovery and Data Mining}{August 2018}{London, United Kingdom}
\acmYear{2018}
\copyrightyear{2018}

\acmArticle{4}
\acmPrice{15.00}

\begin{document}
	\title{On Discrimination Discovery and Removal \\ in Ranked Data using Causal Graph}
	
	\author{Yongkai Wu}
	\orcid{1234-5678-9012}
	\affiliation{%
		\institution{University of Arkansas}
	}
	\email{yw009@uark.edu}
	
	\author{Lu Zhang}
	\affiliation{%
		\institution{University of Arkansas}
	}
	\email{lz006@uark.edu}
	
	\author{Xintao Wu}
	\affiliation{%
		\institution{University of Arkansas}
	}
	\email{xintaowu@uark.edu}

	\begin{abstract}
		Predictive models learned from historical data are widely used to help companies and organizations make decisions. However, they may digitally unfairly treat unwanted groups, raising concerns about fairness and discrimination. In this paper, we study the fairness-aware ranking problem which aims to discover discrimination in ranked datasets and reconstruct the fair ranking. Existing methods in fairness-aware ranking are mainly based on statistical parity that cannot measure the true discriminatory effect since discrimination is causal. On the other hand, existing methods in causal-based anti-discrimination learning focus on classification problems and cannot be directly applied to handle the ranked data. To address these limitations, we propose to map the rank position to a continuous score variable that represents the qualification of the candidates. Then, we build a causal graph that consists of both the discrete profile attributes and the continuous score. The path-specific effect technique is extended to the mixed-variable causal graph to identify both direct and indirect discrimination. The relationship between the path-specific effects for the ranked data and those for the binary decision is theoretically analyzed. Finally, algorithms for discovering and removing discrimination from a ranked dataset are developed. Experiments using the real dataset show the effectiveness of our approaches.
	\end{abstract}

 	\keywords{Discrimination-aware machine learning, fair ranking, causal graph, direct and indirect discrimination}

	\maketitle
	
	\section{Introduction}

	Discrimination-aware machine learning which aims to construct discrimination-free machine learning models has been an active research area in the recent years. Many works have been conducted to achieve this goal by detecting and removing discrimination/biases from the historical training data \cite{kamiran2009classifying,feldman2015certifying,Zhang2017c,calmon2017optimized,nabi2017fair}, or from the constructed machine learning models \cite{calders2010three,kamishima2011fairness,kamishima2012fairness,zafar2017fairness}. However, most works focus on the classification models built for categorical decisions, especially binary decisions.
	In this paper, we investigate discrimination in ranking models, which are another widely used machine learning models adopted by search engines, recommendation systems, and auction systems, etc. To be more specific, we study the discrimination discovery and removal from the ranked data. A ranked dataset is a combination of the candidate profiles with the permutation of the candidates as the decision. Fairness concerns are raised for the ranking models since biases and discrimination can also be introduced into the ranking.
	
	\setlength{\tabcolsep}{3pt}
	\begin{table}[ht]\small
	\caption{A toy example of ranked data.}
	\centering
	\begin{tabular}{|c|c|c|c|c|c|c|c|c|c|c|}
		\hline
		   ID     & u1 & u2 & u3 & u4 & u5 & u6 & u7 & u8 & u9 & u10 \\ \hline
		  Race    & 1  & 1  & 1  & 1  & 1  & 0  & 0  & 0  & 0  & 0  \\ \hline
		Zip Code  & 1  & 1  & 1  & 1  & 1  & 1  & 0  & 0  & 1  & 0  \\ \hline
		Interview & 1  & 2  & 2  & 4  & 2  & 5  & 4  & 4  & 3  & 2  \\ \hline
		   Edu    & 1  & 2  & 1  & 2  & 4  & 5  & 4  & 5  & 3  & 5  \\ \hline
	\end{tabular}
	\label{tab:example}
\end{table}

\begin{figure}[ht]
	\centering
	\includegraphics[width=2.5in]{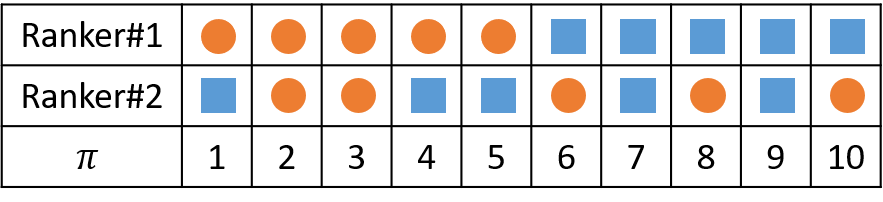}
	\caption{Ranking results produced by two rankers. Blue squares represent the favorable group and red circles represent unfavorable group.}
	\label{fig:misleading}
\end{figure}

	Existing methods \cite{Zehlike,Yang2017} for studying the discrimination discovery and removal from ranked data are mainly based on statistical parity, which means that the demographics of individuals in any prefix of the ranking are identical to the demographics of the whole population. However, it has already been shown in classification that statistical parity does not take into account the fact that part of discrimination is explainable by some non-protected attributes and hence cannot accurately measure discrimination \cite{Dwork2012}. We believe that this observation also holds in the ranked data. Let's consider an illustrative example of ranked data for a company recruiting system shown in Table~\ref{tab:example}.
	The ranked data contains four profile attributes: \textit{race} ($C$), \textit{zip code} ($Z$), \textit{education} ($E$), \textit{interview result} ($I$), where \textit{race} is the protected attribute with a favorable group ($C=1$) and an unfavorable group ($C=0$), and \textit{education} and \textit{interview result} are the objective requirements of getting the job. Assume that there are two rankers, both of which compute the qualification scores to produce the rankings shown in Figure \ref{fig:misleading}. 
	The first ranker, denoted by Ranker\#1, produces qualification scores as an equal-weighted linear combination of two attributes \textit{education} ($E$) and \textit{interview result} ($I$). Intuitively, Ranker\#1 produces a fair ranking since it purely depends on two objective attributes. However, as can be seen, the ranking results do not satisfy statistical parity. On the basis of Ranker\#1, the second ranker, Ranker\#2, further gives a bonus score of $2$ for the favorable group (i.e., $C=1$). The usage of the protected attribute explicitly results in the unfavorable treatment to the protected group ($C = 0$). Nevertheless, the ranking results satisfy statistical parity as two race groups are well-mixed in equal proportion. This example shows that statistical parity may produce misleading conclusions regarding discrimination.

To address the limitation of the statistical parity-based methods, the causal graph-based discrimination detection and removal methods have been recently proposed by Zhang et al. \cite{Zhang2017c}. It shows that the correlation between the protected attribute and the decision is a nonlinear combination of the direct discrimination, the indirect discrimination, as well as the explainable effect. The path-specific effect technique has been used to capture the causal effects passing through different paths. However, this work focuses on binary classification. In ranking systems, the decisions are given in term of a permutation of a series of unique, concatenating integers which cannot be treated as regular random variables. Thus, the methods in \cite{Zhang2017c} cannot be applied directly to deal with ranked data.

In this paper, we employ the causal graph to solve the fair ranking problem by adopting a continuous variable called \emph{score} instead of the ranking positions to represent the qualifications of individuals in the rank. We use the Bradley-Terry model \cite{Bradley1952} to obtain a reasonable mapping from ranking positions to scores. We then construct the causal graph from the individuals' profiles and scores, a mix of categorical and continuous data. One challenge here is that traditional causal graph construction and inference are limited to the single-datatype situations where the variables are all discrete (e.g., the causal Bayesian network) or all continuous (e.g., the linear Gaussian model). To address this challenge, we associate the score with a set of Conditional Gaussian (CG) distributions instead of the Conditional Probability Table (CPT). Then, we extend the path-specific effect technique to our mixed-variable causal graph for capturing direct and indirect discrimination. We also derive a relationship between the path-specific effects for the ranked data and those for the binary decision, assuming that binary decision is obtained based on certain cut-off point imposed on the ranking. Algorithms for detecting discrimination in the causal graph, as well as for removing rank biases from the data are developed. Finally, we conduct experiments using real datasets to show the effectiveness of our methods. The results show that our methods can correctly detect and remove both direct and indirect discrimination with relatively small data utility loss, while the statistical parity based methods neither correctly identify discrimination nor successfully mitigate discrimination.

The rest of this paper is organized as follows. Section~\ref{sec:preliminaries} gives the notations and backgrounds that will be used in the paper. Section~\ref{sec:modeling} models direct and indirect discrimination in a ranked data as causal effects. Sections~\ref{sec:algorithm} develops the discrimination discovery and removal approaches based on the causal graph. Experiment settings and results are reported in Section~\ref{sec:experiment}. Finally, the related work and the conclusions are summarized in Sections~\ref{sec:relatedwork} and Section~\ref{sec:conclusions}.

	\section{Preliminaries}
	\label{sec:preliminaries}

	Throughout the paper, attributes are denoted by an uppercase letter, e.g.: $X$; sets of attributes are denoted by a bold uppercase letter, e.g.: $\mathbf{X}$. A value of a certain attribute is denoted by a lowercase letter, e.g.: $x$; a set of values of a set attributes is denoted by a bold lowercase letter, e.g.: $\mathbf{x}$. The domain space of an attribute is denoted by $\mathfrak{X}_X$.
	The domain space of an attribute set is a Cartesian product of the domain spaces of its elements, denoted by $\mathfrak{X}_\mathbf{X} = \prod_{ X \in \mathbf{X}} \mathfrak{X}_X$.

	A causal graph \cite{Pearl2009a} is a DAG $\mathcal{G}=(\mathbf{V},\mathbf{A})$ where $\mathbf{V}$ is a set of nodes and $\mathbf{A}$ is a set of edges. Each node in the graph represents an attribute. Each edge, denoted by an arrow ``$\rightarrow$'' pointing from the cause to the effect, represents the direct causal relationship. The path that traces arrows directed from one node $X$ to another node $Y$ is called the causal path from $X$ to $Y$.
	For any node $X$, its parents are denoted by $Pa(X)$, and its children are denoted by $Ch(X)$. Usually, the \emph{local Markov condition} is assumed to be satisfied, which means that each node is independent of its non-descendants conditional on all its parents. Each child-parent family in the graph is associated with a deterministic function
	\begin{equation*}
	x = f_{X}(Pa(X),\varepsilon_{X}), \quad X\in \mathbf{V},
	\end{equation*}
	where $\varepsilon_{X}$ is an arbitrarily distributed random disturbance. This functional characterization of the child-parent relationship leads to the conditional probability distribution that characterizes the graph, i.e., $P(x|Pa(X))$. When all variables are discrete, $P(x|Pa(X))$ is denoted by a conditional probability table (CPT).

	Inferring causal effects in the causal graph is performed through interventions, which simulate the physical interventions that fix the values of some attributes to certain constants. Specifically, an intervention that fixes the values a set of variables $\mathbf{X}\subseteq \mathbf{V}$ to constants $\mathbf{x}$ is mathematically formalized as $do(\mathbf{X}=\mathbf{x})$ or simply $do(\mathbf{x})$. The post-intervention distribution of all other attributes $\mathbf{Y}=\mathbf{V}\backslash \mathbf{X}$, denoted by $P(\mathbf{Y}=\mathbf{y}|do(\mathbf{X}=\mathbf{x}))$ or simply $P(\mathbf{y}|do(\mathbf{x}))$, can be calculated using the truncated factorization formula \cite{Koller2009}
	\begin{equation}\label{eq:do1}
	P(\mathbf{y}|do(\mathbf{x})) = \prod_{Y\in \mathbf{Y}}P(y|Pa(Y))\delta_{\mathbf{X}=\mathbf{x}},
	\end{equation}
	where $\delta_{\mathbf{X}=\mathbf{x}}$ means assigning attributes in $\mathbf{X}$ involved in the term ahead with the corresponding values in $\mathbf{x}$. As a result, the \emph{total causal effect} of $\mathbf{X}$ on $\mathbf{Y}$ is assessed by comparing the difference between the post-intervention distributions under two different interventions $do(\mathbf{x}')$ and $do(\mathbf{x}'')$. A common measure of the total causal effect is the expected difference as shown in Definition \ref{def:te}. Note that the total causal effect measures the effect of the intervention that is transmitted along all causal paths from $\mathbf{X}$ to $\mathbf{Y}$.
	
	\begin{definition}[Total causal effect]\label{def:te}
		Given a causal graph $\mathcal{G}=(\mathbf{V},\mathbf{A})$ and two disjoint sets of variables $\mathbf{X},\mathbf{Y}\subseteq \mathbf{V}$, the total causal effect of $\mathbf{X}$ on $\mathbf{Y}$ in terms of two interventions $do(\mathbf{x}')$ and $do(\mathbf{x}'')$, denoted by $\mathit{TE}(\mathbf{x}', \mathbf{x}'')$, is given by
		\begin{equation*}
		\mathit{TE}(\mathbf{x}', \mathbf{x}'') = \mathbb{E}\left[ \mathbf{Y}|do(\mathbf{x}') \right] - \mathbb{E}\left[ \mathbf{Y}|do(\mathbf{x}'') \right],
		\end{equation*}
		where $\mathbb{E}[\cdot]$ is the expectation.
	\end{definition}
	
	As an extension to the total causal effect, Avin et al. \cite{Avin2005} proposed the \emph{path-specific effect} that measures the causal effect where the intervention's effect is transmitted along a subset of the causal paths from $\mathbf{X}$ to $\mathbf{Y}$. Denote a subset of causal paths by $\pi$, and denote by $P(\mathbf{Y}|do(\mathbf{x}'|_{\pi}))$ the post-intervention distribution of $\mathbf{Y}$ with the intervention's effect transmitted along $\pi$.
	Based on that, the $\pi$-specific effect is given by Definition \ref{def:pse}.
	\begin{definition}[Path-specific effect]\label{def:pse}
		Given a causal graph $\mathcal{G}=(\mathbf{V},\mathbf{A})$, two disjoint sets of variables $\mathbf{X},\mathbf{Y}\subseteq \mathbf{V}$, and a subset of causal paths $\pi$, the $\pi$-specific effect of $\mathbf{X}$ on $\mathbf{Y}$ in terms of two interventions $do(\mathbf{x}')$ and $do(\mathbf{x}'')$, denoted by $\mathit{SE}_{\pi}(\mathbf{x}', \mathbf{x}'')$, is given by
		\begin{equation*}
		\mathit{SE}_{\pi}(\mathbf{x}', \mathbf{x}'') = \mathbb{E}\left[ \mathbf{Y}|do(\mathbf{x}'|_{\pi}) \right] - \mathbb{E}\left[ \mathbf{Y}|do(\mathbf{x}'') \right].
		\end{equation*}
	\end{definition}
	
	In \cite{Avin2005}, it is pointed out the condition under which the path-specific effect can be estimated from the observed data, known as identifiability of the path-specific effects. How to deal with the unidentifiable situation is discussed in \cite{Zhang2018a}. In \cite{Shpitser2013}, Shpitser et al. gave the method for calculating the identifiable path-specific effect $P(\mathbf{Y}|do(\mathbf{x}'|_{\pi}))$. These strategies are readily to be applied to our methods. 

	\section{Modeling Direct and Indirect Discrimination in Ranked Data}
	\label{sec:modeling}

    In this section, we study how to model direct and indirect discrimination in a ranked dataset as the causal effect.
		    We consider a ranked dataset $\mathcal{D}$ consisting of $N$ individuals with a protected attribute $C$, several non-protected attributes $\mathbf{Z}=\{Z_{1},\cdots,Z_{j},\cdots\}$, and a rank permutation $\pi$ as the decision. There is a subset of attributes $\mathbf{R}\subseteq \mathbf{Z}$ that may cause indirect discrimination, referred to as the \emph{redlining attributes}.	
				We assume all attributes are categorical.
		We further make two reasonable assumptions: 1) the protected attribute $C$ has no parent; and 2) the score $S$ has no child. The first one is due to the fact that the protected attribute is usually an inherent nature of human beings, and the second one is because the score is the output of a ranking system.

	\subsection{Building Causal Graph for Ranked Data}\label{sec:modeling:building}
	A rank permutation is a series of unique, concatenating integers that cannot be treated as normal categorical random variables.
	We use the Bradley-Terry model to map the ranking positions in a ranked data to the continuous scores. A Bradley-Terry model $\mathcal{M}$ assigns each individual $i$ a score $s_{i}$ ($s_{i}\in \mathbb{R}$) to indicate the qualification preference of individual. Generally, a larger score represents a better qualification. The difference between the scores of two individuals $i,j$ corresponds to the log-odds of the probability $p_{i,j}$ that individual $i$ is ranked before individual $j$ in the rank, i.e.,
	\begin{equation*}
	    s_i - s_j = \log \frac{p_{ij}}{ 1 - p_{ij} }.
	\end{equation*}
    Equivalently, solving for $p_{ij}$ yields
    \begin{equation*}
        p_{ij} = \frac{e^{s_i}}{ e^{s_i} + e^{s_j} }.
    \end{equation*}
	On the other hand, the probability of any rank permutation $\pi$ given a Bradley-Terry model $\mathcal{M}$ is proportional to the product of the probability $p_{i,j}$ of all preference pairs subject to $\pi$, i.e.,
	\begin{equation*}
	P(\pi | \mathcal{M}) \propto \prod_{(i,j): \pi_i <\pi_j} {p_{ij}}.
	\end{equation*}
	Thus, the logarithm likelihood of the Bradley-Terry model $\mathcal{M}$ given the observed rank permutation $\pi$ is given by $\mathcal{L}(\mathcal{M}|\pi) = -\log{P(\pi | \mathcal{M})}$. As a result, the optimal Bradley-Terry model that best fits the observed rank permutation $\pi$ can be obtained by minimizing $\mathcal{L}(\mathcal{M}|\pi)$ as the loss function. Wu et al. \cite{Wua} proved that the loss function is convex and could be efficiently optimized with gradient descent.

	After obtaining the score $S$ using the Bradley-Terry model, we build a causal graph for variables $C$, $\mathbf{Z}$ and $S$.
	We first adopt the PC-algorithm for learning the structure of the causal graph. Since there exist both discrete and continuous variables, different conditional independence testing methods can be adopted, such as chi-square test for discrete variables, partial correlation matrix for continuous variables, and conditional Gaussian likelihood ratio test for mixed variables.	
	Then, for parameterizing the causal graph, we treat discrete and continuous variables in different ways. For discrete variables $C$ and $\mathbf{Z}$ (we can extend our method to the situation where some profile attributes are continuous), each of them is associated with a Conditional Probability Table (CPT). The conditional probabilities can be estimated from data using standard statistical estimation techniques (like the maximum likelihood estimation). For continuous score $S$, it is associated with the Conditional Gaussian (CG) distributions instead of the CPT. Let $\mathbf{Q} = Pa(S)\backslash \{C\}$. For each value assignment $c,\mathbf{q}$ of parents of $S$, there is a CG distribution whose mean and variance are based on $c,\mathbf{q}$. Thus, the CG distribution of $S$ is given by
	\begin{equation*}
	P(s|c,\mathbf{q}) = \mathcal{N}(\mu_{c,\mathbf{q}}, \sigma_{c,\mathbf{q}}^{2}).
	\end{equation*}
	Finally, we fit each CG distribution $\mathcal{N}(\mu_{c,\mathbf{q}}, \sigma_{c,\mathbf{q}}^{2})$ to the scores of all candidates with $C=c$ and $\mathbf{Q}=\mathbf{q}$ using standard statistical estimation techniques.

	As an example, Figure \ref{fig:toy} shows a causal graph of the toy example presented in the Introduction. Each of $C,Z,E,I$ is associated with a CPT representing the conditional probability given the parents, and $S$ is associated with a set of CG distribution where the mean and the variance are based on its parents, the other four variables.

	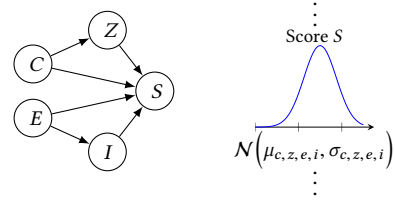
\begin{figure}[ht]
		\hspace*{0cm}
		\centering
		\begin{tikzpicture}[node distance=1.0cm and 1.0cm, mynode/.style={draw,circle,text width=0.1cm,align=center},scale=0.6]
		\node[mynode] at (0,0.4*1.5) (r) {\small $C$};
		\node[mynode] at (0,-0.4*1.5) (e) {\small $E$};
		\node[mynode] at (1.1*1.5,0.9*1.5) (z) {\small $Z$};
		\node[mynode] at (1.1*1.5,-0.9*1.5) (i) {\small $I$};
		\node[mynode] at (1.8*1.5,0) (s) {\small $S$};
		\path 	(r) edge[-latex] (z)
		(r) edge[-latex] (s)
		(z) edge[-latex] (s)
		(e) edge[-latex] (i)
		(e) edge[-latex] (s)
		(i) edge[-latex] (s);
		\node[right=0.6cm of s] 
		{
			\begin{tikzpicture}[scale=0.8]
			    \node at (0.8,2) {\vdots};
				\begin{axis}[every axis plot post/.append style={mark=none,domain=-2:3,samples=50,smooth},
						axis x line=bottom, 
						axis y line=none, 
						xtick=,
						xticklabels={,,},
						title = {Score $S$},
						title style = {yshift=-3ex,},
						xlabel={\large $\mathcal{N} \Big( \mu_{c,z,e,i}, \sigma_{c,z,e,i} \Big)$ },
						x label style = {yshift=2.5ex,},
						enlargelimits=upper] 
					\addplot {gauss(1,0.75)};
				\end{axis}
				\node at (0.8,-0.8) {\vdots};
			\end{tikzpicture}
		};
		
	\end{tikzpicture}
	
	\caption{Causal graph of the toy example involving: \textit{race} ($C$), \textit{zip code} ($Z$), \textit{education} ($E$), \textit{interview result} ($I$), and \textit{score} ($S$).}
	\label{fig:toy}
\end{figure}

\subsection{Quantitative Measurement}

Now we show how direct and indirect discrimination in a ranked data can be quantitatively measured based on the causal graph we build. It is known that discrimination is a causal effect of the protected attribute on the decision. We first give the quantitative measure of the total causal effect of protected attribute $C$ on score $S$ as shown in Theorem \ref{thm:te}.

\begin{theorem}\label{thm:te}
	The total causal effect is given by
	\begin{equation}\label{eq:te}
	\mathit{TE}(c^{+},c^{-}) = \sum_{\mathbf{q} \in \mathfrak{X}_\mathbf{Q}} \left( \mu_{c^{+},\mathbf{q}}P(\mathbf{q}|c^{+}) - \mu_{c^{-},\mathbf{q}}P(\mathbf{q}|c^{-})  \right)
	\end{equation}
\end{theorem}

\begin{proof}
	According to Definition \ref{def:te}, total causal effect is given by
	\begin{equation*}
	\begin{split}
	& \mathit{TE}(c^{+},c^{-}) = \mathbb{E}\left[ S|do(c^{+}) \right] - \mathbb{E}\left[ S|do(c^{-}) \right] \\
	& = \int s\cdot P(s|do(c^{+}))ds - \int s\cdot P(s|do(c^{-}))ds.
	\end{split}
	\end{equation*}
	According to Eq. \eqref{eq:do1}, we have
	\begin{equation*}
	\begin{split}
	& P(s|do(c^{+})) = \sum_{\mathbf{z} \in \mathfrak{X}_\mathbf{Z}} P(s,\mathbf{z}|do(c^{+})) \\
	& = \sum_{\mathbf{z} \in \mathfrak{X}_\mathbf{Z}} P(s|c^{+},\mathbf{q}) \prod_{Z_{j}\in \mathbf{Z}} P(z_{j}|Pa(Z_{j}))\delta_{C=c^{+}}.
	\end{split}
	\end{equation*}
	It can be shown that
	\begin{equation}\label{eq:temp}
	\prod_{Z_{j}\in \mathbf{Z}} P(z_{j}|Pa(Z_{j}))\delta_{C=c^{+}} = P(\mathbf{z}|c^{+}).
	\end{equation}
	In fact, if we sort all nodes in $\mathbf{Z}$ according to the topological ordering as $\{Z_{1},\cdots,Z_{j},\cdots\}$, we can see that all parents of each node $Z_{j}$ are before it in the ordering. In addition, since $C$ has no parent, it must be $Z_{j}$'s non-descendant; since $E$ has no child, it cannot be $Z_{j}$'s parent. Thus, based on the local Markov condition, we have $P(z_{j} | Pa(Z_{j}))=P(z_{j} | c^{+},z_{1},\cdots,z_{j-1} )$. According to the chain rule we obtain $P(\mathbf{z} | c^{+})$.
	Thus, it follows that
	\begin{equation*}
	\begin{split}
	& P(s|do(c^{+})) = \sum_{\mathbf{z} \in \mathfrak{X}_\mathbf{Z}} P(s|c^{+},\mathbf{q}) P(\mathbf{z}|c^{+}) \\
	& = \!\!\! \sum_{\mathbf{q} \in \mathfrak{X}_\mathbf{Q}} \!\!\! P(s|c^{+},\mathbf{q}) \sum_{\mathbf{Z}\backslash \mathbf{Q}} P(\mathbf{z}|c^{+}) = \sum_{\mathbf{q} \in \mathfrak{X}_\mathbf{Q}} \!\!\! P(s|c^{+},\mathbf{q}) P(\mathbf{q}|c^{+}).
	\end{split}
	\end{equation*}
	As a result, we have
	\begin{equation*}
	\begin{split}
	& \int s\cdot P(s|do(c^{+}))ds = \int s\cdot \sum_{\mathbf{q} \in \mathfrak{X}_\mathbf{Q}} P(s|c^{+},\mathbf{q}) P(\mathbf{q}|c^{+}) ds \\
	& = \sum_{\mathbf{q} \in \mathfrak{X}_\mathbf{Q}} P(\mathbf{q}|c^{+}) \int s P(s|c^{+},\mathbf{q})ds = \sum_{\mathbf{q} \in \mathfrak{X}_\mathbf{Q}} \mu_{c^{+},\mathbf{q}}P(\mathbf{q}|c^{+}).
	\end{split}
	\end{equation*}
	Hence, the theorem is proven.
\end{proof}

In \cite{Zhang2017c}, the authors show that in the single-type causal graph, total causal effect generally cannot correctly measure either direct discrimination or indirect discrimination, which should be modeled as the path-specific effects. By adopting similar strategy, we capture direct discrimination by the causal effect transmitted via the direct edge from $C$ to $S$, and capture indirect discrimination by the causal effect transmitted via the paths that pass through redlining attributes. Formally, define $\pi_{d}$ as the path set that contains only $C\rightarrow S$, and define $\pi_{i}$ as the path set that contains all causal paths which are from $C$ to $S$ and pass through $\mathbf{R}$. Then, direct discrimination can be captured by the $\pi_{d}$-specific effect $\mathit{SE}_{\pi_{d}}(\cdot)$, and indirect discrimination can be captured by the $\pi_{i}$-specific effect $\mathit{SE}_{\pi_{i}}(\cdot)$. We extend the method in \cite{Zhang2017c} for computing the path-specific effect from data to our mixed-variable causal graph for computing $\mathit{SE}_{\pi_{d}}(\cdot)$ and $\mathit{SE}_{\pi_{i}}(\cdot)$. The results are shown in Theorem \ref{thm:de&ie}.

\begin{theorem}\label{thm:de&ie}
	The $\pi_{d}$-specific effect $\mathit{SE}_{\pi_{d}}(c^{+},c^{-})$ is given by
	\begin{equation}\label{eq:de}
	\mathit{SE}_{\pi_{d}}(c^{+},c^{-}) = \sum_{\mathbf{q} \in \mathfrak{X}_\mathbf{Q}} \left( \mu_{c^{+},\mathbf{q}} - \mu_{c^{-},\mathbf{q}} \right) P(\mathbf{q}|c^{-}) ,
	\end{equation}
	
	The $\pi_{i}$-specific effect $\mathit{SE}_{\pi_{i}}(c^{+},c^{-})$ is given by
	\begin{equation}\label{eq:ie1}
	\begin{split}
	&\mathit{SE}_{\pi_{i}}(c^{+},c^{-}) = \sum_{\mathbf{z} \in \mathfrak{X}_\mathbf{Z}} \Big( \mu_{c^{-},\mathbf{q}} \prod_{G\in \mathbf{V}_{\pi_{i}}}P(g|c^{+},Pa(G)\backslash \{C\})  \\
	&  \prod_{H\in \bar{\mathbf{V}}_{\pi_{i}}} \!\!\!\ P(g|c^{-},Pa(G)\backslash \{C\}) \!\!\!\! \prod_{O\in \mathbf{Z}\backslash Ch(C)} \!\!\!\! P(o|Pa(O)) \Big) \! - \! \sum_{\mathbf{q} \in \mathfrak{X}_\mathbf{Q}} \left( \mu_{c^{-},\mathbf{q}} P(\mathbf{q}|c^{-}) \right),
	\end{split}
	\end{equation}
	where $\mathbf{V}_{\pi_{i}}$ and $\bar{\mathbf{V}}_{\pi_{i}}$ is obtained by dividing $C$'s children except $S$ based on the above method.
	Eq. \eqref{eq:ie1} can be simplified to
	\begin{equation}\label{eq:ie2}
	\mathit{SE}_{\pi_{i}}(c^{+},c^{-}) = \sum_{\mathbf{q} \in \mathfrak{X}_\mathbf{Q}} \mu_{c^{-},\mathbf{q}} \left( P(\mathbf{q}|c^{+}) - P(\mathbf{q}|c^{-}) \right)
	\end{equation}
	if $\pi_{i}$ contains all causal paths from $C$ to $S$ except the direct edge $C\rightarrow S$.
\end{theorem}

\begin{proof}
	For the $\pi_{d}$-specific effect, according to Definition \ref{def:pse}, we have
	\begin{equation*}
	\begin{split}
	& \mathit{SE}_{\pi_{d}} = \mathbb{E}\left[ S|do(\mathbf{c^{+}}|_{\pi_{d}}) \right] - \mathbb{E}\left[ S|do(\mathbf{c^{-}}) \right] \\
	& = \int s\cdot P(s|do(c^{+}|_{\pi_{d}}))ds - \int s\cdot P(s|do(c^{-}))ds.
	\end{split}
	\end{equation*}
	
	In the above equation, $P(s|do(c^{-}))$ can be computed according to the truncated factorization formula \eqref{eq:do1}. To compute $P(s|do(c^{+}|_{\pi_{d}}))$, we follow the steps in \cite{Shpitser2013}. First, express $P(s|do(c^{+}|_{\pi_{d}}))$ as the truncated factorization formula. Then, divide the children of $C$ into two disjoint sets $\mathbf{V}_{\pi_{d}}$ and $\bar{\mathbf{V}}_{\pi_{d}}$. Let $\mathbf{V}_{\pi_{d}}$ contains $C$'s each child $V$ where edge $C\rightarrow V$ is a segment of a path in $\pi_{d}$; let $\bar{\mathbf{V}}_{\pi_{d}}$ contains $C$'s each child $V$ where either $V$ is not included in any path from $C$ to $S$, or edge $C\rightarrow V$ is a segment of a path not in $\pi_{d}$. Finally, replace values of $C$ with $c^{+}$ for the terms corresponding to nodes in $\mathbf{V}_{\pi}$, and replace values of $C$ with $c^{-}$ for the terms corresponding to nodes in $\bar{\mathbf{V}}_{\pi_{d}}$.

	Following the above procedure, we obtain
	\begin{equation*}
	P(s|do(c^{+}|_{\pi_{d}})) = \!\! \sum_{\mathbf{z} \in \mathfrak{X}_\mathbf{Z}} \!\! P(s|c^{+},\mathbf{q}) \prod_{Z_{i}\in \mathbf{Z}} \!\! P(z_{i}|Pa(Z_{i}))\delta_{C=c^{-}}.
	\end{equation*}
	By using Eq. \eqref{eq:temp}, it follows that
	\begin{equation*}
	P(s|do(c^{+}|_{\pi_{d}})) = \sum_{\mathbf{q} \in \mathfrak{X}_\mathbf{Q}} P(s|c^{+},\mathbf{q}) P(\mathbf{q}|c^{-}),
	\end{equation*}
	which leads to Eq. \eqref{eq:de} in the theorem.
	
	For the $\pi_{i}$-specific effect, following the above procedure similarly we can obtain
	\begin{equation*}
	\begin{split}
	& P(s|do(c^{+}|_{\pi_{i}})) = \sum_{\mathbf{z} \in \mathfrak{X}_\mathbf{Z}} \Big(  P(s|c^{-},\mathbf{q}) \prod_{G\in \mathbf{V}_{\pi_{i}}} \!\! P(g|c^{+},Qa(G)) \\
	& \prod_{H\in \bar{\mathbf{V}}_{\pi_{i}}} \!\!\!\ P(g|c^{-},Pa(G)\backslash \{C\}) \!\!\!\! \prod_{O\in \mathbf{Z}\backslash Ch(C)} \!\!\!\! P(o|Pa(O)) \Big),
	\end{split}
	\end{equation*}
	which leads to Eq. \eqref{eq:ie1}. If $\pi_{i}$ contains all causal paths from $C$ to $S$ except the direct edge, it means that $\mathbf{V}_{\pi_{i}} = Ch(C)\backslash \{S\}$, and $\bar{\mathbf{V}}_{\pi_{i}} = \emptyset$. Thus, it follows that
	\begin{equation*}
	\begin{split}
	& P(s|do(c^{+}|_{\pi_{i}})) = \!\!\! \sum_{\mathbf{z} \in \mathfrak{X}_\mathbf{Z}} \!\!\! P(s|c^{-},\mathbf{q}) \prod_{Z \in \mathbf{Z}} \!\!\! P(z|Pa(Z)\backslash \{C\})\delta_{C=c^{+}} \\
	& = \sum_{\mathbf{q} \in \mathfrak{X}_\mathbf{Q}} P(s|c^{-},\mathbf{q}) P(\mathbf{q}|c^{+}),
	\end{split}
	\end{equation*}
	which leads to Eq. \eqref{eq:ie2}.
	Hence, the theorem is proven.
\end{proof}

Theorems \ref{thm:te} and \ref{thm:de&ie} present the quantitative measurement of the total causal effect as well as the $\pi_{d}$ and $\pi_{i}$-specific effects. The following proposition reveals the relationship among $\mathit{TE}(\cdot)$, $\mathit{SE}_{\pi_{d}}(\cdot)$ and $\mathit{SE}_{\pi_{i}}(\cdot)$. It shows that the indirect (discriminatory) effect is equal to the total causal effect plus the ``reversed'' direct (discriminatory) effect.

\begin{proposition}\label{thm:rel}
	If $\pi_i$ contains all causal paths from $C$ to $S$ except the direct edge $C\rightarrow S$, we have
	\[	\mathit{SE}_{\pi_{i}}(c^{+},c^{-}) = \mathit{TE}(c^{+},c^{-}) + \mathit{SE}_{\pi_{d}}(c^{-},c^{+}). \]
\end{proposition}

\begin{proof}
	The proof can be directly obtained from Eq. \eqref{eq:te} and \eqref{eq:ie2}.
\end{proof}

\subsection{Relationship between Ranking and Binary Decision}
In the earlier work \cite{Zhang2017c}, we have derived the $\pi_d$ and $\pi_i$-specific effects of the protected attribute $C$ on a binary decision attribute $E$ with positive decision $e^{+}$ and negative decision $e^{-}$ (denoted by $\mathit{SE}_{\pi_d}^{E}(\cdot)$ and $\mathit{SE}_{\pi_i}^{E}(\cdot)$ for distinguishing with the path-specific effects derived for ranked data in this paper). Assume that the decision is made based on a cut-off point $\theta$ of the score. Then an interesting question is to ask, given a discrimination-free rank, whether a binary decision made based on the cut-off point $\theta$ is also discrimination free. Answering this question needs to derive a relationship between $\mathit{SE}_{\pi}(\cdot)$ and $\mathit{SE}_{\pi}^{E}(\cdot)$.
In this subsection, we derive such relationships under the condition that $\forall \mathbf{q}$, $\theta \geq \mu_{c^{+},\mathbf{q}}\geq \mu_{c^{-},\mathbf{q}}$ and $\sigma_{c^{+},\mathbf{q}}=\sigma_{c^{-},\mathbf{q}}=\sigma$. We first obtain the formulas of $\mathit{SE}_{\pi_d}^{E}(\cdot)$ and $\mathit{SE}_{\pi_i}^{E}(\cdot)$ using the cut-off point $\theta$.

\begin{lemma}
	Given the causal graph based on score $S$, and a cut-off point $\theta$ for determining a binary decision $E$, we have
	\begin{equation}\label{eq:sede}
	\mathit{SE}_{\pi_d}^{E}(c^{+},c^{-}) = \sum_{\mathbf{q} \in \mathfrak{X}_\mathbf{Q}} \frac{1}{2} \left( \erf(\frac{\theta-\mu_{c^{-},\mathbf{q}}}{\sqrt{2}\sigma}) - \erf(\frac{\theta-\mu_{c^{+},\mathbf{q}}}{\sqrt{2}\sigma}) \right) P(\mathbf{q}|c^{-}),
	\end{equation}
	\begin{equation}\label{eq:seie}
	\mathit{SE}_{\pi_i}^{E}(c^{+},c^{-}) = \sum_{\mathbf{q} \in \mathfrak{X}_\mathbf{Q}} \frac{1-\erf(\frac{\theta-\mu_{c^{-},\mathbf{q}}}{\sqrt{2}\sigma})}{2} \Delta_{\mathbf{q}}.
	\end{equation}
\end{lemma}

\begin{proof}
	Since $\theta$ is a cut-off point, we have $P(e^{+}|c^{+},\mathbf{q})=P(s\geq \theta|c^{+},\mathbf{q})$ and $P(e^{+}|c^{-},\mathbf{q})=P(s\geq \theta|c^{-},\mathbf{q})$. According to the CDF of the Gaussian distribution, we have
	\begin{equation*}
	P(e^{+}|c^{+},\mathbf{q}) = \frac{1-\erf(\frac{\theta-\mu_{c^{+},\mathbf{q}}}{\sqrt{2}\sigma})}{2}, \quad P(e^{+}|c^{-},\mathbf{q}) = \frac{1-\erf(\frac{\theta-\mu_{c^{-},\mathbf{q}}}{\sqrt{2}\sigma})}{2}.
	\end{equation*}
	The lemma is proven by substituting $P(e^{+}|c^{+},\mathbf{q})$ and $P(e^{+}|c^{-},\mathbf{q})$ in the formulas of $\mathit{SE}_{\pi_d}^{E}$ and $\mathit{SE}_{\pi_i}^{E}$ in \cite{Zhang2017c} with the above expressions.
\end{proof}

Then we present two lemmas to show the properties of $\erf(\cdot)$.

\begin{lemma}\label{thm:erf1}
	For any $x_1\geq x_2\geq 0$, we have
	\begin{equation*}
	\frac{1}{2} \left( \erf(x_1) - \erf(x_2) \right) \leq \erf(\frac{x_1-x_2}{2}).
	\end{equation*}
\end{lemma}

\begin{proof}
	Since $\erf(x)$ ($x\geq 0$) is concave and $\erf(0)=0$, we have
	\begin{equation*}
	\frac{\erf(x_2)}{x_2} \geq \frac{\erf(x_1)}{x_1} \quad \Longrightarrow \quad \frac{x_2}{2x_1}\erf(x_1) \leq \frac{1}{2}\erf(x_2)
	\end{equation*}
	which follows that
	\begin{equation*}
	\left( \frac{1}{2} - \frac{x_2}{2x_1} \right) \erf(x_1) \geq \frac{1}{2}\erf(x_1) - \frac{1}{2}\erf(x_2).
	\end{equation*}
	Again, since $\erf(x)$ ($x\geq 0$) is concave and $\erf(0)=0$, we have
	\begin{equation*}
	\left( \frac{1}{2} - \frac{x_2}{2x_1} \right) \erf(x_1) \leq  \erf(\frac{x_1}{2}-\frac{x_2}{2}).
	\end{equation*}
	Combining the above two inequalities, the lemma is proven.
\end{proof}

\begin{lemma}\label{thm:erf2}
	For any $t\geq 0$, when $0\leq x\leq t$, we have
	\begin{equation*}
	\alpha_t x \leq \erf(x) \leq \alpha_t x + \beta_t,
	\end{equation*}
	where
	\begin{equation*}
	\alpha_t = \frac{\erf(t)}{t}, \quad \beta_t = \erf(\sqrt{\ln\frac{2t}{\sqrt{\pi}\erf(t)}}) - \frac{\erf(t)}{t} \sqrt{\ln\frac{2t}{\sqrt{\pi}\erf(t)}}.
	\end{equation*}
\end{lemma}

\begin{proof}
	It is obvious that $\erf(x) \geq \alpha_t x$ ($0\leq x\leq t$). Then, $\beta_t$ is obtained by calculating the tangent line with the slope $\alpha_t$ of $\erf(x)$.
\end{proof}

Based on the above results, the following two theorems characterize the relationship between $\mathit{SE}_{\pi}$ and $\mathit{SE}_{\pi}^{E}$.

\begin{theorem}
	Given the causal graph based on score $S$ and an arbitrary cut-off point $\theta$, if for the ranking derived from the score we have
	\begin{equation*}
	\mathit{SE}_{\pi_{d}}(c^{+},c^{-}) \leq \frac{2\sqrt{2}(\tau - \beta_t) \sigma}{\alpha_t},
	\end{equation*}
	for the binary decision derived from the score we must have $\mathit{SE}_{\pi_i}^{E}(c^{+},c^{-}) \leq \tau$
	where
	\begin{equation*}
	t = \max_{\mathbf{q}} \left\{ \frac{\mu_{c^{+},\mathbf{q}}-\mu_{c^{-},\mathbf{q}}}{2\sqrt{2}\sigma} \right\}.
	\end{equation*}
\end{theorem}

\begin{proof}
	Let $x_1 = \frac{\theta-\mu_{c^{-},\mathbf{q}}}{\sqrt{2}\sigma}$, $x_2 = \frac{\theta-\mu_{c^{+},\mathbf{q}}}{\sqrt{2}\sigma}$, according to Lemma \ref{thm:erf1} we have
	\begin{equation*}
	\frac{1}{2} \left( \erf(x_1) - \erf(x_2) \right) \leq \erf(\frac{x_1-x_2}{2}) = \erf(\frac{\mu_{c^{+},\mathbf{q}}-\mu_{c^{-},\mathbf{q}}}{2\sqrt{2}\sigma}).
	\end{equation*}
	According to Lemma \ref{thm:erf2} it follows that
	\begin{equation*}
	\erf(\frac{\mu_{c^{+},\mathbf{q}}-\mu_{c^{-},\mathbf{q}}}{2\sqrt{2}\sigma}) \leq \alpha_t \frac{\mu_{c^{+},\mathbf{q}}-\mu_{c^{-},\mathbf{q}}}{2\sqrt{2}\sigma} + \beta_t.
	\end{equation*}
	Combining the above inequality with Eq. \eqref{eq:sede}, we have
	\begin{equation*}
	\begin{split}
	& \mathit{SE}_{\pi_d}^{E} \leq \sum_{\mathbf{q} \in \mathfrak{X}_\mathbf{Q}} \left( \alpha_t \frac{\mu_{c^{+},\mathbf{q}}-\mu_{c^{-},\mathbf{q}}}{2\sqrt{2}\sigma} + \beta_t \right) P(\mathbf{q}|c^{-}) \\
	& = \frac{\alpha_t}{2\sqrt{2}\sigma} \mathit{SE}_{\pi_d} + \beta_t \leq \tau.
	\end{split}
	\end{equation*}
\end{proof}

\begin{theorem}
	Given the causal graph based on score $S$ and an arbitrary cut-off point $\theta$, if for the ranking derived from the score we have
	\begin{equation*}
	\mathit{SE}_{\pi_{i}}(c^{+},c^{-}) \leq \frac{2\sqrt{2}(\tau - c) \sigma}{\alpha_t},
	\end{equation*}
	for the binary decision derived from the score we must have $\mathit{SE}_{\pi_i}^{E}(c^{+},c^{-}) \leq \tau$
	where
	\begin{equation*}
	\begin{split}
	& t = \max_{\mathbf{q}} \left\{ \frac{ \max\{s\}-\mu_{c^{-},\mathbf{q}}}{\sqrt{2}\sigma} \right\}, \\
	& c = \frac{1}{2}- \!\!\! \sum_{\mathbf{q}:\Delta_{\mathbf{q}}\geq 0} \left( \frac{\alpha_t \max_{\mathbf{q}} \{ \mu_{c^{+},\mathbf{q}} \} }{\sqrt{2}} \right) - \!\!\! \sum_{\mathbf{q}:\Delta_{\mathbf{q}}< 0} \left( \frac{\alpha_t}{2\sqrt{2}}+\beta_t \right).
	\end{split}
	\end{equation*}
\end{theorem}

\begin{proof}
	According to Lemma \ref{thm:erf2} we have
	\begin{equation*}
	\erf(\frac{\theta-\mu_{c^{-},\mathbf{q}}}{\sqrt{2}\sigma}) \geq \erf(\frac{\max_{\mathbf{q}} \{ \mu_{c^{+},\mathbf{q}} \} -\mu_{c^{-},\mathbf{q}}}{\sqrt{2}\sigma}) \geq \alpha_t \frac{\max_{\mathbf{q}} \{ \mu_{c^{+},\mathbf{q}} \} -\mu_{c^{-},\mathbf{q}}}{\sqrt{2}\sigma},
	\end{equation*}
	\begin{equation*}
	\erf(\frac{\theta-\mu_{c^{-},\mathbf{q}}}{\sqrt{2}\sigma}) \leq \erf(\frac{\max\{s\}-\mu_{c^{-},\mathbf{q}}}{\sqrt{2}\sigma}) \leq \alpha_t \frac{\max\{s\}-\mu_{c^{-},\mathbf{q}}}{\sqrt{2}\sigma} + \beta_t.
	\end{equation*}
	Combining the above inequalities with Eq. \eqref{eq:seie}, we have
	\begin{equation*}
	\begin{split}
	& \mathit{SE}_{\pi_{i}}^{E} = \sum_{\mathbf{q}:\Delta_{\mathbf{q}}\geq 0} \frac{1-\erf(\frac{\theta-\mu_{c^{-},\mathbf{q}}}{\sqrt{2}\sigma})}{2} \Delta_{\mathbf{q}} + \sum_{\mathbf{q}:\Delta_{\mathbf{q}}< 0} \frac{1-\erf(\frac{\theta-\mu_{c^{-},\mathbf{q}}}{\sqrt{2}\sigma})}{2} \Delta_{\mathbf{q}} \\
	& \leq \frac{1}{2} \! + \!\!\!\! \sum_{\mathbf{q}:\Delta_{\mathbf{q}}\geq 0} \!\! \alpha_t \frac{\mu_{c^{-},\mathbf{q}}-\max_{\mathbf{q}} \{ \mu_{c^{+},\mathbf{q}} \} }{2\sqrt{2}\sigma } \Delta_{\mathbf{q}} + \!\!\! \sum_{\mathbf{q}:\Delta_{\mathbf{q}}< 0} \!\! (\alpha_t \frac{\mu_{c^{-},\mathbf{q}}-\max\{s\}}{2\sqrt{2}\sigma}-\beta_t) \Delta_{\mathbf{q}} \\
	& = \frac{\alpha_t}{2\sqrt{2}\sigma} \mathit{SE}_{\pi_{i}} + c \leq \tau.
	\end{split}
	\end{equation*}
\end{proof}

\section{Discovery and Removal Algorithms}
\label{sec:algorithm}
We develop the discrimination discovery and removal algorithms based on the derived $\pi_{d}$ and $\pi_{i}$-specific effects.
Since the values of $\mathit{SE}_{\pi_{d}}(c^{+},c^{-})$  and $\mathit{SE}_{\pi_{i}}(c^{+},c^{-})$ can be arbitrarily large, we give the criterion of direct and indirect discrimination in terms of relative difference. We require that the ratio of $\mathit{SE}_{\pi_{d}}(c^{+},c^{-})$  and $\mathit{SE}_{\pi_{i}}(c^{+},c^{-})$ over the expected score of the non-protected group, i.e., $\mathbb{E}[S|c^{+}]$, is smaller than a given threshold $\tau$. For example, the Equality and Human Rights Commission (EHRC) 
consider $0.05$ as a significant threshold for the gender pay gap.  By defining the discrimination measures
\begin{equation*}
\mathit{DE}_{d}(c^{+},c^{-}) = \frac{\mathit{SE}_{\pi_{d}}(c^{+},c^{-})}{\mathbb{E}[S|c^{+}]}
\end{equation*}
and
\begin{equation*}
\mathit{DE}_{i}(c^{+},c^{-}) = \frac{\mathit{SE}_{\pi_{i}}(c^{+},c^{-})}{\mathbb{E}[S|c^{+}]},
\end{equation*}
the criterion of discrimination is shown below. To avoid reverse discrimination, we also similarly define $\mathit{DE}_{d}(c^{+},c^{-})$ and $\mathit{DE}_{i}(c^{+},c^{-})$. Then, we give the criterion of discrimination as follows.
\begin{criterion}\label{thm:criterion}
	Given a user-defined threshold $\tau$, direct discrimination exists if either $\mathit{DE}_{d}(c^{+},c^{-}) > \tau$ or $\mathit{DE}_{d}(c^{-},c^{+}) > \tau$ holds, and indirect discrimination exists if either $\mathit{DE}_{i}(c^{+},c^{-}) > \tau$ or $\mathit{DE}_{i}(c^{-},c^{+}) > \tau$ holds.
\end{criterion}


Based on the above analysis, we develop the algorithm for discovering discrimination in a rank, referred to as \emph{FDetect}, as shown in Algorithm \ref{alg:dd}.
Once direct or indirect discrimination is detected, the discriminatory effects need to be eliminated before the ranked data is used for training or sharing. We propose a path-specific-effect-based Fair Ranking (\emph{FRank}) algorithm to remove both discrimination from the ranked data and reconstruct a fair ranking. We first modify the score distributions so that the causal graph contains no discrimination, and then reconstruct a fair ranking based on the modified causal graph. As shown in Theorem \ref{thm:de&ie}, the discriminatory effect only depends on the means of the score distributions. Hence we only need to modify the means of the score.

\begin{algorithm2e}
	\SetAlgoVlined
	\SetKwInOut{Input}{Input}
	\SetKwInOut{Output}{Output}
	
	\Input{Ranked dataset $\mathcal{D}$, protected attribute $C$, user-defined parameter $\tau$.}
	\Output{Direct/indirect discrimination $judge_d$, $judge_i$.}
	$judge_d = judge_i = false$\;
	Derive the score $S$ using the Bradley-Terry model\;
	Build the causal graph for $S$ and attributes in $\mathcal{D}$\;
	Compute $\mathit{DE}_{d}(\cdot)$ according to Theorem \ref{thm:de&ie}\;
	\If{$\mathit{DE}_{d}(c^{+},c^{-})> \tau $ $\|$ $\mathit{DE}_{d}(c^{-},c^{+})> \tau$} {
		$judge_d = true$\;
	}
	Divide $C$'s children except $S$ into $\mathbf{V}_{\pi_{i}}$ and $\bar{\mathbf{V}}_{\pi_{i}}$\; 
	Compute $\mathit{DE}_{i}(\cdot)$ according to Theorem \ref{thm:de&ie}\;
	\If{$\mathit{DE}_{i}(c^{+},c^{-})> \tau$ $\|$ $\mathit{DE}_{i}(c^{-},c^{+})> \tau$} {
		$judge_i = true$\;
	}
	\Return{$[judge_d,judge_i]$}\;
	\caption{\emph{FDetect}}
	\label{alg:dd}
\end{algorithm2e}

To maximize the utility during the modification process, we minimize the distance between the original score distributions and the modified score distributions, as measured by the Bhattacharyya distance \cite{Bhattacharyya}. Specifically, for each score distribution $\mathcal{N}(\mu_{c,\mathbf{q}},\sigma^{2}_{c,\mathbf{q}})$, denote the modified distribution by $\mathcal{N}(\mu_{c,\mathbf{q}}',\sigma^{2}_{c,\mathbf{q}})$. The Bhattacharyya distance between the two distributions is given by
\begin{equation*}
D_{B} \!\! = \!\! -\ln \!\! \int \!\! \sqrt{\mathcal{N}(\mu_{c,\mathbf{q}},\sigma^{2}_{c,\mathbf{q}}) \mathcal{N}(\mu_{c,\mathbf{q}}',\sigma^{2}_{c,\mathbf{q}})} ds = \frac{ \left( \mu_{c,\mathbf{q}}-\mu_{c,\mathbf{q}}' \right)^{2} }{8\sigma^{2}_{c,\mathbf{q}}}.
\end{equation*}
We define the objective function as the sum of the Bhattacharyya distances for all score distributions. As a result, we obtain the following quadratic programming problem with $\mu_{c,\mathbf{q}}$ as the variables.

\begin{equation*}
\begin{split}
\textrm{minimize} & \qquad \sum_{ c \in \mathfrak{X}_C, \mathbf{q} \in \mathfrak{X}_\mathbf{Q} }  \frac{ \left( \mu_{c,\mathbf{q}}-\mu_{c,\mathbf{q}}' \right)^{2} }{\sigma^{2}_{c,\mathbf{q}}} \\
\textrm{subject to} & \qquad \mathit{DE}_{d}(c^{+},c^{-})\leq \tau, \quad \mathit{DE}_{d}(c^{-},c^{+})\leq \tau, \\
& \qquad \mathit{DE}_{i}(c^{+},c^{-})\leq \tau, \quad \mathit{DE}_{i}(c^{-},c^{+})\leq \tau.
\end{split}
\end{equation*}

After obtaining the modified score distribution by solving the quadratic programming problem, we reconstruct a fair ranking as follows. Consider the individuals with the same profile $c,\mathbf{q}$, i.e., $\forall i, c_{i}=c, \mathbf{q}_{i}=\mathbf{q}$. For each individual $i$, the new score $s_{i}'$ is regenerated from the new CG distribution $\mathcal{N}(\mu'_{c,\mathbf{q}}, \sigma_{c,\mathbf{q}}^{2})$ at the same percentile as the score $s_{i}$ in the original distribution. Specifically, since $s_{i}=\mu_{c,\mathbf{q}}+\rho \sigma_{c,\mathbf{q}}$ and $s_{i}'=\mu_{c,\mathbf{q}}'+\rho \sigma_{c,\mathbf{q}}$ where $\rho$ is the value from the standard normal distribution for the percentile, we have $s_{i}' = s_{i}+(\mu_{c,\mathbf{q}}'-\mu_{c,\mathbf{q}})$. Finally, we re-rank all individuals according to the descending order of their new scores. Since the new scores contain no discrimination, so does the new rank. The procedure is shown in Algorithm \ref{alg:fr}, referred to as \emph{FRank}.

\begin{algorithm2e}
	\SetAlgoVlined
	\SetKwInOut{Input}{Input}
	\SetKwInOut{Output}{Output}
	
	\Input{Ranked dataset $\mathcal{D}$, protected attribute $C$, user-defined parameter $\tau$.}
	\Output{Modified dataset $\mathcal{D}^{*}$.}
	\If{$\textit{PSE-DD}(\mathcal{D},C,\tau)==[false, false]$} {
		\Return{}\;
	}
	Obtain the modified distributions of $S$ by solving the quadratic programming problem\;
	\ForEach{$c,\mathbf{q}$} {
		\ForEach{$i:c_{i}=c,\mathbf{q}_{i}=\mathbf{q}$} {
			$s_{i}' = s_{i}+(\mu_{c,\mathbf{q}}'-\mu_{c,\mathbf{q}})$\;
		}
	}
	Compute the new rank of each individual according to the descending order of $S$, and replace the rank in $\mathcal{D}$ with the new one to obtain $\mathcal{D}^{*}$\;
	\Return{$\mathcal{D}^{*}$}\;
	\caption{\emph{FRank}}
	\label{alg:fr}
\end{algorithm2e}

The computational complexity of our discovery and removal algorithms depends on how efficiently to derive the score $S$ using  Bradley-Terry model. \cite{Wua} proved that the likelihood function is convex and the optimal solution can be efficiently obtained using gradient descent. The complexity also depends on the complexities of building the causal graph and computing the path-specific effect.
Many researches have been devoted to improving the performance of network construction \cite{kalisch2007estimating,tsamardinos2003scaling,aliferis2010local} and probabilistic inference in causal graphs \cite{heckerman1994new,heckerman1996causal}.
The complexity analysis can be found in these related literature.

\section{Experiments}
\label{sec:experiment}

\subsection{Experimental Setup}
For evaluating the proposed methods, we use a real world dataset, the German Credit \cite{Lichman2013}, which is also used in previous works \cite{Yang2017,Zehlike}.
The German Credit dataset consists of 1000 individuals with 20 attributes applying for loans. Due to the small sample size, we only select 8 attributes in our experiments including \textit{age}, \textit{dependant}, \textit{duration}, \textit{housing}, \textit{job}, \textit{property}, \textit{purpose}, \textit{residence}. We treat \textit{age} as the protected attribute, \textit{housing} as the redlining attribute.
We employ the weighted-sum ranking strategy proposed in \cite{Yang2017,Zehlike} to generate two ranked datasets, denoted by \textbf{D1} and \textbf{D2}. The weighted sum is computed through a weighted linear summation of certain attributes, and then all candidates are ranked according to the weighted sum. In \textbf{D1}, all attributes are summed up with equal weight, while in \textbf{D2}, the summation is for all attributes except \textit{age}.
We also use another ranked data \textbf{D} where the ranking is directly based on an original attribute \textit{credit amount}.
After that, we derive the continuous qualification scores from each ranked dataset using the Bradley-Terry model.
The causal graphs are then constructed using the open source software TETRAD \cite{tetrad} and parameterized as described in Section \ref{sec:modeling:building}.
The constructed causal graph for \textbf{D} is shown in Figure \ref{fig:graph}. The causal graphs for the other two datasets are not shown due to space constraints. The quadratic programming is solved using CVXOPT \cite{Dahl2006}. The discrimination threshold $\tau$ is set as 0.05 for both direct and indirect discrimination.

For comparison we involve the statistical parity-based discrimination discovery and removal algorithms proposed by Yang et al. \cite{Yang2017} and Zehlike et al. \cite{Zehlike}. For discrimination discovery, Yang et al. \cite{Yang2017} proposed three set-based discrimination measures called $\mathbf{rRD}$, $\mathbf{rND}$, and $\mathbf{rKL}$ to compute the difference between the protected group and the whole dataset in terms of \emph{risk difference}, \emph{risk ratio}, and \emph{Kullback-Leibler distance}. They compute the values of difference at several discrete points (\eg top-10, top-20, $\cdots$) and sum up all values with the logarithmic discounts. All measures are normalized to 0-1 range (0 is the most fair value and 1 is the least fair value). Since they don't provide any criterion for discrimination discovery, we simply use $0.05$ as the threshold for all three measures.  Zehlike et al. \cite{Zehlike} proposed an adjusted fairness condition (\emph{FairCon}) that requires the minimum number of protected candidates in every prefix of the ranking list. For discrimination removal, Yang et al. \cite{Yang2017} proposed a fair data generator (\emph{FairGen}) that manipulates the permutation according to the user-defined preference $f$. For example, if $f=0.05$, all the candidates are well mixed in equal proportion at every prefix; if $f=1$, the candidates from the protected group are ranked at the bottom. Zehlike et al. \cite{Zehlike} proposed discrimination removal methods, \emph{FA*IR}, to select the most qualified candidate from the corresponding group at every prefix in order to satisfy the adjusted fairness condition.

To evaluate the data utility of all removal approaches,  we adopt two widely used metrics, the Spearman's footrule distance (\emph{SFD}) and the Kendall's tau distance (\emph{KTD}) \cite{Kumar2010}. The Spearman's footrule distance (\emph{SFD}) measures the total element-wise displacement between the modified permutation and the original one.  The Kendall's tau distance (\emph{KTD}) measures the total number of pairwise inversions between the two permutations. For both of the distance metrics, the larger values indicate more data utility loss.

\subsection{Discrimination Discovery}

\begin{figure}
	\centering
	\includegraphics[width=0.7\linewidth]{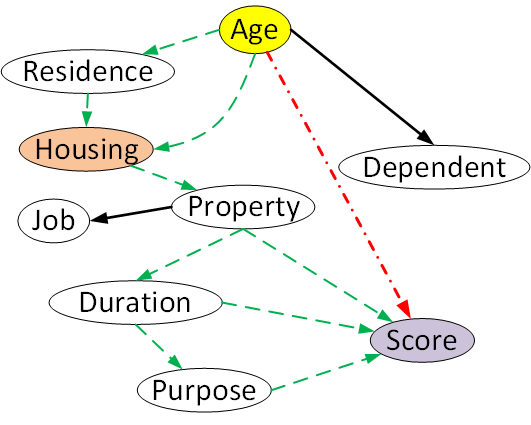}
	\caption{The causal graph of \textbf{D}. The yellow node Age is the protected attribute, the orange node Housing is the redline attribute, and the purple node Score is the decision attribute. The red dash-dot line captures the direct discrimination from Age to Score, and the green dashed line captures the indirect discrimination throgh Housing.}
	\label{fig:graph}
\end{figure}

We quantify the strength of direct and indirect discrimination using our method \emph{FDetect} for all three ranked datasets. The results are shown in Table \ref{tab:discovery}. For dataset \textbf{D1}, all attributes including the protected attribute are used for ranking directly. Thus, the ground-truth is that both direct and indirect discrimination occurs in this dataset. Our method obtains $\mathit{DE}_{d}(c^{+},c^{-}) = 0.231$ and $\mathit{DE}_{i}(c^{+},c^{-}) = 0.055$, showing that both direct and indirect discrimination are correctly identified. For dataset \textbf{D2}, since we use all the other attributes except the protected attribute in the ranking process, the ground-truth is that the indirect discrimination occurs but the direct discrimination doesn't. Our method shows $\mathit{DE}_{d}(c^{+},c^{-}) = 0.026$ and $\mathit{DE}_{i}(c^{+},c^{-}) = 0.061$, which is also consistent with the ground-truth. For dataset \textbf{D}, we don't have the ground-truth. Our method obtains that $\mathit{DE}_{d}(c^{+},c^{-}) = 0.005$ and $\mathit{DE}_{i}(c^{+},c^{-}) = 0.013$. The results imply that neither direct discrimination nor indirect discrimination exists in this dataset.

The statistical parity-based methods $\mathbf{rRD}, \mathbf{rND}, \mathbf{rKL}$ and \emph{FairCon} cannot distinguish direct and indirect discrimination. We directly report the results produced by these methods as shown in Table \ref{tab:discovery}. For \textbf{D1}, the method proposed by Yang et al. shows that $\mathbf{rRD} = 0.204$, $\mathbf{rND} = 0.440$, and $\mathbf{rKL} = 0.590$, while Zehlike's \emph{FairCon} shows that the third position doesn't satisfy the minimum fair requirement. For \textbf{D2}, Yang's method shows that $\mathbf{rRD} = 0.068$, $\mathbf{rND} = 0.225$, and $\mathbf{rKL} = 0.346$, while \emph{FairCon} reports that the 5-th position doesn't satisfy the fair requirement. All methods conclude discrimination for both dataset, which kind of match our conclusions. However, for \textbf{D}, Yang's method shows that $\mathbf{rRD} = 0.008$, $\mathbf{rND} = 0.070$, and $\mathbf{rKL} = 0.109$, where three values make the contradictory conclusions: there is no discrimination according to $\mathbf{rRD}$ but $\mathbf{rND}$ and $\mathbf{rKL}$ report significant discrimination. \emph{FairCon} shows that the ranking cannot satisfy the fair requirement at the 20-th position. All methods cannot obtain the results that are consistent with ours, implying that they may produce incorrect or misleading conclusions.

\setlength{\tabcolsep}{3.5pt}

\begin{table}[ht]
	\centering
	\caption{Comparison of discrimination discovery methods on the direct discrimination. The second column represents the ground-truth for direct and indirect discrimination. }
	\label{tab:discovery}
	\begin{tabular}{|c?c?c|c?c?c|c|c?}
		\hline
		               & \textit{Truth} & $\mathit{DE}_{d}$ & $\mathit{DE}_{i}$ & \textit{FairCon} & \textbf{rRD} & \textbf{rND} & \textbf{rKL} \\ \hline
		 \textbf{D1}   & Y/Y            & 0.231             & 0.055             &        3rd         &    0.204     & 0.44         &     0.59     \\ \hline
		 \textbf{D2}   & N/Y            & 0.026             & 0.061             &        5th         &    0.068     & 0.225        &    0.346     \\ \hline
		  \textbf{D}   & -              & 0.005             & 0.013             &        20th        &    0.008     & 0.07         &    0.109     \\ \hline
	\end{tabular}
\end{table}

\subsection{Discrimination Removal}

We perform \emph{FRank} to remove discrimination and reconstruct fairly ranked datasets with neither direct nor indirect discrimination. Our theoretical results guarantee that there is no discrimination after modification. For comparison, we also execute \emph{FairGen} \cite{Yang2017} and \emph{FA*IR} \cite{Zehlike}. After removing discrimination, we further apply \emph{FDetect} to evaluate whether the newly-generated data achieves truly discrimination-free. The results of three removal methods are shown in Table \ref{tab:comparison}. As can be seen, our method \emph{FRank} removes both direct and indirect discrimination precisely. However, \emph{FairGen} and \emph{FA*IR} cannot achieve discrimination-free. \emph{FairGen} removes neither direct nor indirect discrimination. It even introduces more discrimination to \textbf{D}. \emph{FA*IR} can mitigate part of direct discrimination, but fails to remove indirect discrimination.

\setlength{\tabcolsep}{5pt}

\begin{table}[ht]
	\centering
	\caption{Discrimination and data utility measured on the new ranked data produced by \emph{FairGen}, \emph{FA*IR}, and our \emph{FRank}. Values
		violating the discrimination criterion are marked in bold.}
	\label{tab:comparison}
	\begin{tabular}{|c|c|r|r|r|r|}
		\hline
		             Data               &    Methods     & $\mathit{DE}_{d}$ & $\mathit{DE}_{i}$ & \emph{KTD} & \emph{SFD} \\ \hline
		 \multirow{3}{*}{\textbf{D1}}   &  \emph{FRank}  &           $0.050$ &           $0.050$ &    $24602$ &    $72938$ \\ \cline{2-6}
		                                & \emph{FairGen} &  $\textbf{0.234}$ &  $\textbf{0.064}$ &    $11150$ &    $44600$ \\ \cline{2-6}
		                                &  \emph{FA*IR}  &  $\textbf{0.077}$ &  $\textbf{0.066}$ &    $13882$ &    $55528$ \\ \hline
		 \multirow{3}{*}{\textbf{D2}}   &  \emph{FRank}  &           $0.029$ &           $0.050$ &     $5851$ &    $18090$ \\ \cline{2-6}
		                                & \emph{FairGen} &  $\textbf{0.246}$ &  $\textbf{0.060}$ &    $19483$ &    $77934$ \\ \cline{2-6}
		                                &  \emph{FA*IR}  &           $0.022$ &  $\textbf{0.061}$ &      $231$ &      $924$ \\ \hline
		  \multirow{3}{*}{\textbf{D}}   &  \emph{FRank}  &           $0.005$ &           $0.013$ &        $0$ &        $0$ \\ \cline{2-6}
		                                & \emph{FairGen} &  $\textbf{0.250}$ &           $0.012$ &    $20806$ &    $83226$ \\ \cline{2-6}
		                                &  \emph{FA*IR}  &           $0.003$ &           $0.013$ &      $143$ &      $572$ \\ \hline
	\end{tabular}
\end{table}

We adopt the Spearman's footrule distance (\emph{SFD}) and the Kendall's tau distance (\emph{KTD}) to evaluate the data utility loss when mitigating the discrimination. As can be seen from the last two columns of Table \ref{tab:comparison}, our method \emph{FRank} incurs relatively small data utility loss, but \emph{FairGen} suffers large data utility loss while not achieving discrimination-free. Although \emph{FA*IR} introduces quite a small data utility loss, it fails to mitigate indirect discrimination. It is worth pointing out that there is no direct or indirec discrimination in $\textbf{D}$ so our \emph{FRank} doesn't result in any distortion. On the  contrary, \emph{FairGen} leads to too much utility loss.

We also examine how the data utility varies with different values of the discrimination threshold $\tau$. We perform \emph{FRank} on \textbf{D1} and vary the threshold $\tau$ for \emph{FRank} from $0.00$ to $0.25$ for evaluating how much data utility loss is incurred. In Table~\ref{tab:result}, we can see that both the Spearman's footrule distance (\emph{SFD}) and the Kendall's tau distance (\emph{KTD}) decrease with the increase of $\tau$, which means that less utility loss is incurred with a larger threshold. This observation is consistent with our analysis since the larger $\tau$, the more relaxed the constraints in \emph{FRank}.

\setlength{\tabcolsep}{3.5pt}

\begin{table}[ht]
	\centering
	\caption{Comparison of \emph{FRank} with varied $\tau$.}
	\label{tab:result}
	\begin{tabular}{|c|r|r|r|r|r|r|}
		\hline
		$\tau$                         & $0.00$   & $0.05$  & $0.10$  & $0.15$  & $0.20$  & $0.25$  \\ \hline
		$\mathit{DE}_{d}(c^{+},c^{-})$ & $0.000$  & $0.050$ & $0.100$ & $0.150$ & $0.200$ & $0.231$ \\ \hline
		$\mathit{DE}_{i}(c^{+},c^{-})$ & $0.000$  & $0.050$ & $0.055$ & $0.055$ & $0.055$ & $0.055$ \\ \hline
		\emph{SFD}                     & $43490$  & $24602$ & $14370$ & $9041$  & $3444$  & $0$     \\ \hline
		\emph{KTD}                     & $123626$ & $72938$ & $45636$ & $28652$ & $11054$ & $0$     \\ \hline
	\end{tabular}
\end{table}

\section{Related Work}
\label{sec:relatedwork}

Fairness-aware learning is an active research area in machine learning and data mining. Many methods have been proposed for constructing discrimination-free machine learning models, which either based on data preprocessing or model tweaking. Data preprocessing methods \cite{kamiran2009classifying,feldman2015certifying,Zhang2017c,calmon2017optimized,nabi2017fair} modify the historical training data to remove discrimination before it is used for learning a model. Model tweaking methods \cite{calders2010three,kamishima2011fairness,kamishima2012fairness,zafar2017fairness} require some tweak or adjustment of the constructed machine learning models.
Fair ranking is an emerging topic in fairness-aware learning. Current works in fair ranking are mainly based on the statistical parity. In \cite{Zehlike}, it is required that a preset proportion of protected individuals that must be maintained in each prefix of the ranking for the rank to be fair. However, many works (e.g., \cite{Zhang2017c}) have shown that statistical parity alone is insufficient as a general notion of fairness.

Recently, several studies have been devoted to analyzing discrimination from the causal perspective \cite{Zhang2016,Zhang2017a,Zhang2017b,Zhang2017c,Zhang2017d,Zhang2018a,bonchi2017exposing,nabi2017fair}. In \cite{bonchi2017exposing}, the authors proposed a framework based on the Suppes-Bayes causal network and developed several random-walk-based methods to detect different types of discrimination. However, it is unclear how the number of random walks is related to practical discrimination metrics. In addition, the construction of the Suppes-Bayes causal network is impractical with the large number of attribute-value pairs. In this work we adopt the causal graph used in \cite{Zhang2016,Zhang2017a,Zhang2017b,Zhang2017c,Zhang2017d,Zhang2018a,nabi2017fair}. A causal graph is a probabilistic graph model widely used for causation representation, reasoning and inference \cite{Pearl2009a}.
The limitation of these works is that they focus on the classification problems and cannot be applied directly to the fair ranking problem. This is because in their models, the decision of each individual is treated as an independent random variable, but the ranking positions of different individuals are correlated. In this paper we address above limitations and develop the causal-based fair ranking algorithms.

In data science, it is well-studied in data mining how to model a ranking using a continuous score space \cite{Marden1996}. Several models, such as the Plackett-Luce model \cite{Plackett1975,Luce1959}, the Mallows model \cite{Mallows1957} and the Bradley-Terry model \cite{Bradley1952}, are widely used in this field. In this work, we adopt the Bradley-Terry model to characterize the ranked data and obtain the continuous scores from the ranks.

\section{Conclusions and Future Work}
\label{sec:conclusions}

In this paper, we studied the problem of discovering discrimination in a rank and reconstructing a fair rank if discrimination is detected. We made use of the causal graph to capture the biases in the rank as the causal effect. To address the limitation of the existing single-datatype causal graph, we modeled the ranking positions using a continuous score, and built the causal graph for the profile attributes as well as the score. Then, we extended the path-specific effect technique to the mixed-variable causal graph, which is used to quantitatively measure direct and indirect discrimination in the ranked data. We also theoretically analyzed the relationship between the path-specific effects for the ranked data and those for the binary decision. Based on that, we developed an algorithm for discovering both direct and indirect discrimination, as well as an algorithm to reconstruct a fair rank from the causal graph. The experiments using the German Credit dataset showed that our methods correctly measure the discrimination in the rank and reconstruct a rank that does not contain either direct or indirect discrimination, while the statistical parity-based method may obtain incorrect and misleading results.

In Theorem \ref{thm:de&ie} we assume that the $\pi_i$-specific effect is identifiable from the data. In some cases, the $\pi_i$-specific effect is not able to be computed from the data due to the inherent unidentifiability of the path-specific effect \cite{Avin2005}. In our another work \cite{Zhang2018a}, we have discussed how to deal with this situation, referred to as the unidentifiable situation, and developed a lower and upper bound to the unidentifiable path-specific effect. Similar ideas can be adopted to deal with unidentifiable situation for ranked data. We leave this part for future work.

{\noindent \bf Repeatability:} Our software together with the datasets used in this paper are available at http://tiny.cc/fair-ranking.

\begin{acks}
This work was supported in part by NSF 1646654. 
\end{acks}

\bibliographystyle{ACM-Reference-Format}
\bibliography{lib,ijcai18}

\end{document}